\documentclass[letterpaper]{article} 
\usepackage{aaai24}  
\usepackage{times}  
\usepackage{helvet}  
\usepackage{courier}  
\usepackage[hyphens]{url}  
\usepackage{graphicx} 
\usepackage{natbib}  
\usepackage{caption} 
\usepackage{algorithm}
\usepackage{amsmath}
\usepackage{bibentry}

\usepackage{mathtools}
\usepackage{amsthm, amssymb}
\newtheorem{theorem}{Theorem}
\newtheorem{definition}{Definition}
\newtheorem{corollary}{Corollary}

\usepackage[noend]{algpseudocode}
\usepackage{newfloat}
\usepackage{listings}
\DeclareCaptionStyle{ruled}{labelfont=normalfont,labelsep=colon,strut=off} 
\floatstyle{ruled}
\newfloat{listing}{tb}{lst}{}
\floatname{listing}{Listing}
\title{ReCePS: Reward Certification for Policy Smoothed Reinforcement Learning}
\author{
    Ronghui Mu\textsuperscript{\rm 1},
    Leandro Soriano Marcolino \textsuperscript{\rm 2},
    Tianle Zhang \textsuperscript{\rm 1},
    Yanghao Zhang \textsuperscript{\rm 1},\\
    Xiaowei Huang \textsuperscript{\rm 1},
   Wenjie Ruan \textsuperscript{\rm 1}\footnote{Corresponding Author}
}
\affiliations{
    \textsuperscript{\rm 1}Department of Computer Science, University of Liverpool, Liverpool, L69 3BX, UK\\
    \textsuperscript{\rm 2}School of Computing \& Communication, Lancaster University, Lancaster, LA1 4YW, UK\\

  \{r.mu, x.huang, t.zhang,y.zhang\}@liverpool.ac.uk, l.marcolino@lancaster.ac.uk, w.ruan@trustai.uk 
}

\begin{document}

\maketitle

\begin{abstract}

Reinforcement Learning (RL) has achieved remarkable success in safety-critical areas, but it can be weakened by adversarial attacks. Recent studies have introduced "smoothed policies" to enhance its robustness. Yet, it is still challenging to establish a provable guarantee to certify the bound of its total reward. Prior methods relied primarily on computing bounds using Lipschitz continuity or calculating the probability of cumulative reward above specific thresholds. However, these techniques are only suited for continuous perturbations on the RL agent's observations and are restricted to perturbations bounded by the $l_2$-norm. To address these limitations, this paper proposes a general {\em black-box} certification method, called {\bf ReCePS}, which is capable of directly certifying the cumulative reward of the smoothed policy under various $l_p$-norm bounded perturbations. Furthermore, we extend our methodology to certify perturbations on action spaces. Our approach leverages $f$-divergence to measure the distinction between the original distribution and the perturbed distribution, subsequently determining the certification bound by solving a convex optimisation problem. We provide a comprehensive theoretical analysis and run sufficient experiments in multiple environments. Our results show that our method not only improves the certified lower bound of the mean cumulative reward but also demonstrates better efficiency than state-of-the-art methods.

\end{abstract}

\section{Introduction}
The utilisation of neural networks in Reinforcement Learning (RL) has achieved remarkable success in 
 safety-critical domains, such as controlling robots and autonomous driving \cite{sallab2017deep,pan2017virtual,johannink2019residual}. Nevertheless, 
 recent research has revealed their vulnerability to the presence of adversarial perturbations \cite{SzegedyZSBEGF13,MadryMSTV17,GoodfellowSS14,mu2021sparse,jin2022enhancing,9186644}. For example, numerous studies 
 have demonstrated that even well-trained RL policies can suffer significant failures when 
 directly perturbing the observations of the RL agent \cite{PattanaikTLBC18,mu2023certified} or in action space \cite{LinHLSLS17}. 
In this regard, it is vital to analyse their robustness before their deployment in safety-critical systems \cite{christiano2016transfer,cheng2019end}.

Various empirical defences have been proposed to defend adversarial attacks in RL systems \cite{manikandan2011measures,pattanaik2017robust}, while it has been demonstrated that even robust models can still be compromised by more advanced attack methods \cite{russo2019optimal}. Hence, there is a need for computing provable guarantees for the trained policy to disrupt the repeated game between attacks and defence, which is referred to as {\em robustness certification}. The majority of studies aiming to provide robustness certification primarily focus on classification tasks, while the certification for RL is seldom touched. 

Compared to classification tasks, certifying RL engages more challenges due to its sequential decision-making nature. To address this obstacle, recent efforts have focused on certifying a "smoothed policy" based randomised smoothing strategies. What sets these approaches apart is that they do not require access to the internal architecture and parameters of the DNNs, which is referred to as \textit{black-box certification}. For instance, \citet{wu2021crop} employed Lipschitz continuity to approximate the final output, but this approach resulted in a loose bound. Instead of directly certifying the lower bound, \cite{kumar2021policy} proposed certifying the reward by breaking it down into several thresholds and then assessing the probability of the cumulative reward staying above a specific threshold. However, this method proved too restrictive (see discussion in Section \ref{sec:com}), leading to a loss of essential information regarding the output reward and resulting in weaker certificates for smoothed policies. Additionally, both approaches can only handle observation perturbations under the $l_2$ norm bound. However, in the real world, adversaries can perturb both observations and actions within the bounds of various perturbation constraints of the $l_p$ norm. 

 Recognising the limitations of current state-of-the-art techniques, this paper focuses on effectively certifying the lower bound of mean utility for a policy under diverse $l_p$-norm perturbations. We present a novel approach based on the generalisation theorem between distributions. By leveraging this theorem, we demonstrate that determining the lower bound of expected utility can be achieved by solving a convex optimisation problem. Doing this enables us to directly compute the lower bound, resulting in enhanced certification outcomes. Additionally, by employing $f$-divergence to quantify the distance between distributions, our approach can be expanded to provide certification for a range of $l_p$-norm bounded perturbations, which includes certifying observation perturbations constrained by the $l_1$-norm and action space perturbations bounded by the $l_0$-norm.

Our contributions can be summarised as: 
\textbf{(i)} We propose a novel methodology to directly certify the cumulative reward of the smoothed policy. This approach uses $f$-divergence to gauge the separation between the original distribution and the perturbed distribution. Subsequently, we calculate the certification bound by solving a convex optimisation problem. \textbf{(ii)} Our method is capable of handling perturbations bounded by both the $l_0$ and $l_1$ norm. This work is the first of its kind to consider the certification of the $l_0$-norm bounded perturbation in the action space. \textbf{(iii)} By comparing our method with the previous approach, we theoretically validate that our method can enhance certified robustness by taking into account the distribution of cumulative rewards during sampling. \textbf{(iv)} Through comprehensive experiments, we demonstrate that our method outperforms the state-of-the-art methods, producing tighter bounds for $l_2$ perturbations. Furthermore, we conduct intensive experiments in various environments to demonstrate the effectiveness of our certification method for $l_1$ perturbations in observation and $l_0$ perturbations in the action space.

\section{Related Work}
Various empirical defences have been proposed to defend adversarial attacks in RL system \cite{manikandan2011measures,pattanaik2017robust,DBLP:conf/iclr/EysenbachL22}, and certain studies have employed adversarial training techniques to enhance policy robustness \cite{shen2020deep, zhang2020robust}. The robustness certification mainly focused on classification tasks. They employ various approaches, such as deterministic methods \cite{ehlers2017formal, wong2018provable, mu20223dverifier, zhang2023reachability}, and probability-based techniques \cite{lecuyer2019certified, cohen2019certified,jin2023randomized，sun2023textverifier}, to establish lower bounds for the classification accuracy in the presence of specific perturbations. Regarding the certification of robustness in RL, \cite{wu2021crop} introduced a method to certify both the reward and action taken at each time step. However, they can only handle non-adaptive adversary in static scenarios. Another study by \cite{kumar2021policy} aimed to compute a lower bound for the mean cumulative reward in the face of adaptive adversaries, using the Neyman-Pearson Lemma. Yet, their method necessitates dividing the reward into multiple thresholds and subsequently determining whether the reward surpasses these thresholds.
In this paper, we present a general technique by solving convex optimization problems, which directly enables the certification of cumulative rewards in the presence of adaptive adversaries, without the need to partition the reward. Broader discussion of related work are in Appendix H.
\section{Preliminaries}
\subsection{Deep Q-Networks (DQNs)}
Markov decision processes (MDP) serve as the foundation of reinforcement learning (RL). It can be formed by a tuple $(\mathcal{S}, \mathcal{A}, P, R, \gamma)$, where $\mathcal{S} \in \mathbb{R}^D$ represents a set of continuous states, $\mathcal{A}$ is a set of discrete actions, $P : \mathcal{S}\times \mathcal{A} \rightarrow \mathcal{P}(\mathcal{S})$ is the transition function, $R: \mathcal{S} \times \mathcal{A} \rightarrow \mathbb{R}$ denotes the reward function, and $\gamma$ is the discount factor $\in [0,1]$. We denote the stationary policy as $\pi : \mathcal{S} \rightarrow \mathcal{P}(\mathcal{A})$. In the state $s_t \in \mathcal{S}$ at the time step $t$, the agent will take action $a_t \in \mathcal{A} \sim \pi(s_t)$ and then move to the next state $s_{t+1} \sim P\left(s_{t}, a_{t}\right)$, receiving the reward $R(s_t,a_t)$. The goal is to learn a policy that maximises expected discounted cumulative reward $\mathbb{E}[\sum_t \gamma^t R(s_t,\pi(s_t)]$. 

Deep Q networks \cite{mnih2013playing} used a neural network to approximate the maximum expected cumulative reward, which is called the action value function (Q function), after taking action $a_t$ at the state $s_t$, $Q(s_t,a_t) = R(s_t,a_t) + \gamma \mathbb{E}[max_aQ(s_{t+1}, a_{t+1})]$. The system consists of two neural networks: one is responsible for updating the network parameters $\theta$, while the other acts as the target network, sharing the same architecture as the first but with fixed parameters. Once the initial network has undergone multiple iterations, its parameters are transmitted to the target network. Additionally, a replay buffer is utilised to store past experience tuples generated through interactions with the environment. These tuples are subsequently fed into the neural network to update the parameters.

\subsection{Policy Smoothed Reinforcement Learning}
This paper focuses on evaluating the finite-step Reinforcement Learning with smoothed policy introduced by \citet{kumar2021policy}. The smoothing policy training process involves incorporating random smoothing noise from a probability distribution into the observation of the input state $s_t$ of the agents at each time step $t$. For a detailed demonstration of the training process, please refer to the Appendix A.
\begin{definition}\label{def:1}
(Smoothed policy) Given policy $\pi$, let $\forall s_{t} \in \mathcal{S}$, the noise vector $\Delta_t \in \mathbb{R}^N$ is $i.i.d$ drawn from the Gaussian distribution $\mathcal{N}(0, \sigma^2 I_N)$, the smoothed policy can be represented as    
\begin{equation}\label{eq:1}
\begin{aligned}
\tilde{\pi}(s_t)&={\pi}\left(s_{t} + \Delta_t\right):=\mathop{\arg\max}\limits_{{a}_t \in \mathcal{A}} \widetilde{Q}^{\pi}\left(s_{t} + \Delta_t , {a}_t\right) \\
\end{aligned}
\end{equation}
\end{definition} 
The objective is to establish a lower bound on the expected sum of rewards obtained within a finite time frame when the observation is perturbed under a constraint. \citet{kumar2021policy} proposed to certify the lower bound of the mean reward based on the technique developed by \cite{kumar2020certifying} using the empirical cumulative distribution function (CDF) of the probability of surpassing a specific threshold. In this paper, we propose a general framework that can find the lower bound of mean cumulative reward directly via optimisation approach.

\section{Methodology}
Our main objective is to certify the lower bound of the mean cumulative reward during the testing in the presence of an adversary intentionally perturbing the agent's observations or actions. In particular, perturbations within the action space constitute a distinct instance of $l_0$ perturbation, which we will explicitly define in Section \ref{sec:l0}. In this section, our primary focus is on perturbations in observations.
\subsection{General Robustness Certification }
 We consider the general adversarial setting in reinforcement learning \cite{kumar2021policy,wu2021crop,huang2017adversarial}, where the observation of the agent is perturbed by $\delta_t$ at each time step $t$. Suppose that the entire sequence of adversarial perturbation is $\delta=\{\delta_1,\delta_2,...\}$ and the overall $l_p$-norm of the perturbation is bounded by $\epsilon$. The adversarial attack aims to reduce the cumulative reward of the smoothed policy by perturbing the observations:
\begin{equation}
\begin{array}{l}\underset{\epsilon}{\min} J_{\epsilon}(\tilde{\pi}):=\underset{t=0}{\sum} \gamma^t R(s_t,{\pi}(s'_t))=\underset{t=0}{\sum} \gamma^t R(s_t,\tilde{\pi}(s_t+\delta_t)),  \\ \text { s.t. }\left\|\delta_{1}, \delta_{2}, \ldots\right\|_p \leq \epsilon \end{array}
\end{equation}
The robust certification for the provably robust RL with policy $\pi$ in the finite-step game is to certify the lower bound of the total reward $\underline{J}$ under the norm bound perturbation $\epsilon$.
\begin{equation}\label{eq:object}
\min _{\epsilon} J_{\epsilon}(\tilde{\pi})\geq \underline{J}, \text { s.t. }\|\delta\|_p \leq \epsilon
\end{equation}

Nevertheless, it is challenging to obtain the exact solution for this worst-case function. In this paper, instead of investigating the adversarial attack within the observation space $s_t \in \mathcal{S}$, we shift our focus to studying it within the domain of probability measures over observations. The smoothed policy can be interpreted as incorporating random inputs by sampling observations from the distribution $\mu(s_t)$. In other words, we define $\tilde{\pi}$ as $\pi_{o_t \sim \mu(s_t)}\left(o_{t}\right)$, where $o_{t}$ represents the sampled observation.

In the reinforcement learning (RL) framework, the perturbation introduced to one step may depend on the current state, previous action, and previous observation, it was shown that incorporating noise into sequential observations within a reinforcement learning (RL) framework does not result in an isometric distribution across the observation space. Therefore, in line with previous research \cite{kumar2021policy}, we expends the complete perturbation budget, denoted as $\delta$, exclusively in the initial coordinate of the initial perturbation vector to simplify the analysis. 

\begin{definition}
    Given a state trajectory $(s_0, s_1,...,s_{T-1})$, we define the \textbf{smoothed} state and observation trajectory as $\tau=(s_0, o_0,s_1,o_1,...,s_{T-1},o_{T-1})$ where $s_{t+1}\sim P(s_t,\pi(o_t))$ and $o_t \sim \mu(s_t)$. Suppose we perturb the observation at each time step by $\delta_t$, the \textbf{perturbed} trajectory can be define as $\tau=(s'_0, o'_0,s'_1,o'_1,...,s'_{T-1},o'_{T-1})$ where $s'_t=s_t + \delta_t$, $s'_{t+1}\sim P(s'_t,\pi(o'_t))$ and $o'_t \sim \mu(s'_t)$. In the following analysis in this paper, we use $\tau \sim \mu(s_t)$ represents smoothed trajectory and $\tau \sim \mu(s'_t)$ denotes the perturbed trajectory.
\end{definition}

The objective in Eq. \ref{eq:object} can be rewritten as:
\begin{equation}\label{eq:3}
\min _{q \in\left\{\mu\left(s^{\prime}_t\right):\left\|\delta\right\| \leq \epsilon\right\} } \underset{\tau \sim q}{J(\tau)}:={\sum}^{T-1}_{t=0} \gamma^t R(s^{\prime}_t,\pi(o'_t)) \geq \underline{J}
\end{equation}
Under the constraint $q \in \mathcal{D}_\epsilon =\left\{\mu\left(s^{\prime}_t\right):\left\|\delta\right\| \leq \epsilon\right\}$, this is an optimisation problem of infinite dimensional over the space of probability distribution $q \in \mathcal{P}(s_t)$. Suppose that the $p=\mu (s_t)$, the divergence constraint can be represented as $\mathcal{D}_\epsilon \subseteq\left\{q: D(q \| p) \leq \epsilon\right\}$, where $D$ is the divergence between two distributions.
\subsection{Divergence Measure}
Therefore, the verification problem can be formulated as an optimisation problem to find the minimum expectation utility of the smoothed policy $\tilde{\pi}$:
\begin{equation}
\min _{q \in \mathcal{D}_{\epsilon}} \underset{\tau \sim q}{\mathrm{E}}[J(\tau)]
 \end{equation}
    where $\mathcal{D}_\epsilon \subseteq\left\{q: D(q \| p) \leq \epsilon\right\}$. To address this issue, we must establish a particular relaxation of the set $\mathcal{D}_\epsilon$, as it may not be convex and cannot be solved directly. Hence, we will consider the sets $D(q \| p)$ that can be easily optimised over. In this paper, we will explore three broad constraint sets of $f$-divergences \cite{csiszar1967information} that can be adapted to evaluate the $l_p$-norm between two distribution.
\begin{definition}($f$-divergence measure \cite{csiszar1967information}). Given the distribution $q$, $p \in \mathcal{P}(s)$, $f$ is a convex function with $f(1)=0$.
    \begin{equation}
D(q||p)_f:=\int f\left(\frac{d q}{d p}\right) d p
\end{equation}
\end{definition}
Notably, $D(q||p) \geq 0$, and $D(q||p) = 0$ if $q=p$. Given the reference distribution $p$ and the $f$-divergence $D(q||p)$, the constraint set under the bound $\epsilon_D \geq 0$ can be defined as
\begin{equation}
\mathcal{D}_{f}=\left\{q \in \mathcal{P}(s): D(q \| p) \leq \epsilon_{D}\right\}
\end{equation}
\subsection{Optimisation Based  Certification Approach}
In this section, we demonstrate our approach to tackling the challenges of solving problem given in Eq.\ref{eq:3}, which is achieved by converting it into a convex optimisation problem. Our methodology is versatile, capable of accommodating various constraint sets $\mathcal{D}$ defined by diverse f-divergences.
The foundational support for constructing the convex optimisation problem is derived from the Generalised Donsker-Varadhan Variational Formula, as presented in Theorem 4.2 of \cite{ben2007old}. This formula plays a pivotal role in establishing a duality connection between the primal and dual forms of the Optimal Control of Expectations (OCE) optimisation problem.

In this paper, we extend the utility of this theorem, adapting it as a fundamental theorem to determine the optimal lower bound for the expected cumulative reward within the framework of RL.

 \begin{corollary}\label{co:1} \footnote{All proofs for Theorems are available in the Appendix.} (Adaptive Generalized Donsker-Varadhan Variational Formula for RL) Given the $f$-divergence, Suppose $q=\mu(s_t'), p=\mu(s_t)$, let $z(\tau)=\frac{q(\tau)}{p(\tau)}$, which represents the probability of the full trajectory. $J(\tau)$ denoting as the cumulative reward of the smoothed policy. Given the convex function $f$ with $f(1)=0$, we have
\begin{equation}\small
\min _{z}\left\{
D(q||p)+\underset{\tau \sim p}{\mathrm{E}}\left[z(\tau)J(\tau)\right]\right\}=\max _{\eta \in \mathrm{R}}\left\{\eta-\underset{\tau \sim p}{\mathrm{E}} [f^{*}(\eta-J(\tau))]\right\},
\end{equation}
where $f^*$ represents the convex conjugate (the Legendre–Fenchel transform) of $f$, $f^*(x)=\underset{y >0}{max}(xy-f(y))$.
\end{corollary}

Building on the Corollary \ref{co:1}, we derive the pivotal theorem in this paper, which serves as the objective for the convex optimisation problem.
\begin{theorem}\label{th:1}
Given the convex function $f$ with $f(1)=0$, and $f^*$ is its convex conjugate. Under the $f$-divergence constraint $D_f(q \| p) \leq \epsilon_D$ with $\epsilon_D \geq 0$, let $z(\tau)=\frac{q(\tau)}{p(\tau)}$, we can solve the following convex optimisation problem to find the optimal value for the expected cumulative reward of the smoothed policy $\tilde{\pi}\left(s\right)$:
\begin{equation}\label{eq:opt}
\max _{\nu>0, \eta \in \mathbb{R}}\left\{\nu\left[\eta-\underset{\tau \sim p}{\mathrm{E}}\left(f^{*}\left(\eta+\epsilon-\frac{J(\tau)}{\nu}\right)\right)\right]\right\}
\end{equation}
\end{theorem}
\begin{proof}
Based on Corollary \ref{co:1}, we can rewrite the formula on the right side:

\begin{equation*}
\begin{array}{l}\underset{\nu>0, \eta \in \mathbb{R}}{\max}\left\{\nu\left[\eta-\underset{\tau \sim p}{\mathrm{E}}\left(f^{*}\left(\eta+\epsilon-\frac{J(\tau)}{\nu}\right)\right)\right]\right\} \\ =\underset{\nu>0, \eta \in \mathbb{R}}{\max}\left\{\nu\left[\eta+\epsilon-\underset{\tau \sim p}{\mathrm{E}}\left(f^{*}\left(\eta+\epsilon-\frac{J(\tau)}{\nu}\right)\right)-\epsilon \right]\right\} \\ 
=\underset{\nu>0}{\max} \left\{\nu\left[\underset{t\in \mathbb{R}}{\max}\left\{t-\underset{\tau \sim p}{\mathrm{E}}\left(f^{*}\left(t-\frac{J(\tau)}{\nu}\right)\right)\right\}-\epsilon \right]\right\} \\ 
=\underset{\nu >0}{\max} \left\{\nu\left[\min _{z}\left\{
D(q||p)+\underset{\tau \sim p}{\mathrm{E}}\left[z(\tau)\frac{J(\tau)}{\nu}\right]\right\}-\epsilon\right]\right\} \\ 
=\underset{\nu>0}{\max}\left\{\underset{z}{\min} \left\{\underset{\tau \sim p}{\mathrm{E}}\left[z(\tau){J(\tau)}\right]+\nu\left(D(q||p)-\epsilon\right)\right\}\right\} \\ 

=\underset{z}{\min} \left\{\underset{\tau \sim p}{\mathrm{E}}\left[z(\tau)J(\tau)\right]\right\}\\
=\underset{\tau \sim q}{\min} \left\{\mathrm{E}\left[J(\tau)\right]\right\}
\end{array}
\end{equation*}
As a result, maximising the problem stated in Eq. \ref{eq:opt} can give the lower bound for the mean cumulative reward.
\end{proof}
To solve the convex optimisation problem in Eq. \ref{eq:opt}. we need to estimate the expected value of $\underset{\tau \sim p}{\mathrm{E}}\left(f^{*}\left(\eta+\epsilon-\frac{J(\tau)}{\nu}\right)\right)$. In the RL system, due to its sequential decision-making nature, calculating the precise expectation of cumulative reward in closed form is challenging. To overcome this obstacle, we employ Monte Carlo Sampling as a means to approximate the cumulative reward. Given a random variable with distribution $p(x)=\mu(x)$, where $x^1,x^2,..., x^M$ are drawn independently and identically from distribution $p(x)$, the expected value of $x$ can be approximated as $\tilde{\mu}(x)= \frac{1}{M} \sum_{i=1}^M x^i$. 

\begin{algorithm}[tb]
\caption{Certified Robustness Bound of the Cumulative Reward}
\label{alg:algorithm2}
\textbf{Input}: Soothed Policy ${\tilde{\pi}}$; action value function $Q^{\tilde{\pi}}$ perturbation bound $\epsilon$; divergence function $f$\\
\textbf{Parameter}: small sampling size $M$; large sampling size $\tilde{M}$; Gaussian distribution parameter $\sigma$;  confidence parameter $\alpha$   
\begin{algorithmic}[1] 
\Function{SmoothingRewardSet}{$M,{Q}^{{\tilde{\pi}}}$}
\For {$m \gets 1,M$}
\State $J_m \gets 0$
\For {$t \gets 1,T$}
    \State Generate noise $\Delta \sim \mathcal{N}(0, \sigma^2 I_D)$ 
    \State $s'_t \gets s_t + \Delta$ 
    \State $a_t \gets \mathop{\arg\max}\limits_{{a}_t \in \mathcal{A}} {Q}^{\tilde{\pi}}\left(s'_{t}, {a}_t\right)$ 
    \State Execute $a_t$ and obtain $R$ and $s_{t+1}$
    \State $J_m \gets J_m + R$
    \EndFor
\EndFor
\State $\mathbf{J}\gets \{J_1,...,J_M\}$
\EndFunction
\State \textbf{return} $\mathbf{J}$
\Function{Optimisation}{$M,{Q}^{\tilde{\pi}},\epsilon, f$}
\State $\mathbf{J} \gets \mathop{SmoothingRewardSet} (M,{Q}^{\tilde{\pi}})$
\State Solving the convex optimisation problem 
$$\max _{\nu>0, \eta \in \mathbb{R}}\left\{\nu\left[\eta-\sum^{M}_{i=1}\left(f^{*}\left(\eta+\epsilon-\frac{J_i}{\nu}\right)\right)\right]\right\}$$
\State \textbf{return} $v^*, \eta^*$
\EndFunction
\Statex \Comment{Compute the lower bound of certified reward}
\State $v^*,\eta^* \gets \mathop{Optimisation}(M,{Q}^{\tilde{\pi}},\epsilon, f)$
\State $\tilde{\mathbf{J}} \gets \mathop{SmoothingRewardSet} (\tilde{M},{Q}^{\tilde{\pi}})$
\State Substitute the $v^*$ and $\eta^*$ to compute $f_i^{*}\left(\eta^*+\epsilon-\frac{\tilde{J}_i}{\nu^*}\right)$ for $i=1,2,...,\tilde{M}$
\State Apply Eq. \ref{eq:noumd} to compute $E_U$ by replacing $h(\tau_i)$ with $f_i^{*}$ .
\State Substitute the expectation term by $E_U$ in Eq. \ref{eq:opt} to obtain the final lower bound of mean reward.
\end{algorithmic}
\end{algorithm}

Our verification procedure is outlined in Algorithm \ref{alg:algorithm2}. Initially, we generate $M$ episodes by introducing noise $\Delta$ to observations, where the noise values $\Delta$ can be independently and identically drawn from a normal distribution $\mathcal{N}(0, \sigma^2 I_D)$. In the context of RL, we sample random noise from a fixed distribution and incorporate noise with observations to observe multiple instances of cumulative reward. Subsequently, we observe $M$ cumulative rewards $(J_1, J_2, ..., J_M)$. These samples are then utilised to solve the jointly convex optimisation problem (Eq. \ref{eq:opt}), which can be efficiently addressed using a powerful tool \cite{diamond2016cvxpy}. Finally, we obtain the optimal values of $\nu$ and $\eta$ through this optimisation process. 

To establish a confident guarantee for the lower bound of the objective function, we need to estimate the maximum value of the expectation term denoted as $E_U \geq \underset{\tau \sim p}{\mathrm{E}}\left(f^{*}\left(\eta+\epsilon-\frac{J(\tau)}{\nu}\right)\right)$. To achieve this approximation, we increase the number of i.i.d. samples, denoted as $\tilde{M}$, drawn from the distribution $p$. These additional samples are used to compute an upper bound with high confidence, centred around the empirical estimation of the expectation term.  In this paper, we employ  the widely recognised Hoeffding’s bound \cite{shivaswamy2010empirical} to approximate the lower bound of this expectation.
\begin{definition}\label{def:3}(Hoeffding’s bound evaluation)
    At each state, randomly sampling perturbed observation i.i.d from distribution $p=\mu(s_t)$. Running M episodes to obtain cumulative rewards, $(J(\tau_1),J(\tau_2),...,J(\tau_M))$, for any $\alpha >0$, we can obtain the provable evaluation bound of the expectation value of $h(\tau)$:
    \begin{equation*}
        {P}(|\frac{1}{M}\sum_i^M h(\tau_i)-\underset{\tau \sim p}{E}(h(\tau))|\leq \zeta ) \leq 1-\alpha
    \end{equation*}
    where $h(\tau_i)=f^{*}\left(\eta+\epsilon-\frac{J(\tau_i)}{\nu}\right)$. 
\end{definition}
The Definition \ref{def:3}  remains valid when $\zeta=\sqrt{\frac{R^2}{2 M} \log \frac{2}{\alpha}}$, where $R$ is the range of $h(\tau)$. Therefore, with confidence at least $1-\alpha$, we have 
\begin{equation}\label{eq:noumd}
    E_U = \frac{1}{M}\sum_i^M h(\tau_i) + \zeta.
\end{equation}


\subsection{Comparison with prior work}\label{sec:com}
Two existing studies concerning the certification of cumulative rewards in RL are \cite{kumar2021policy} and \cite{wu2021crop}. In \cite{kumar2021policy}, the authors tackle certification by forming it as a classification problem, determining whether the cumulative reward surpasses a specific threshold. They utilise integrals to calculate the total expected reward through the cumulative distribution function. Conversely, \citet{wu2021crop} directly certifies the lower boundary of the anticipated reward by exploiting the Lipschitz continuity of the smoothed reward. Nevertheless, in comparison to both our work and the methodology proposed by \citet{kumar2021policy}, this approach demonstrates a relatively weak performance, as indicated by the findings in our experimental results provided in the Appendix I. This divergence in performance could be potentially due to their emphasis on certifying per-step perturbations. Consequently, in this article, we will present a theoretical analysis that elucidates why our method outperforms the approach introduced in \cite{kumar2021policy}.

The problem studied in \cite{kumar2021policy} may be treated as a special case of our methodology. When considering the case of smoothed Gaussian noise i.i.d sampling from distribution $\mathcal{N}(0,\sigma^2I_D)$, the Hockey-Stick divergence can be analytically calculated, which can be expressed as a function of the $l_2$-norm, as established in \cite{dong2021black} (detailed proof in Appendix D). The Hockey-Stick divergence, denoted as $D_{HS,\lambda}(q||p)$, is defined by the function  $f(x)=max(x-\lambda,0)-max(1-\lambda,0)$.

The certification framework proposed in \cite{kumar2021policy} involves working with a reward range and separating it into multiple thresholds. We can reformulate our optimisation objective to recover the certification result presented in \cite{kumar2021policy} as following:
\begin{theorem}
   (CDF-based optimisation framework) Suppose the reward range is (a, b) and considering $n$ thresholds, denoted as $a<g_1 < g_2 < ... < g_n<b$. Let $\theta_i$ represent the lower bound of probability that the total reward obtained by the smoothed policy is above the threshold $g_i$. Let f* be the convex conjugate of hockey-stick divergences, we have 
   \begin{equation*}
    \begin{array}{l}
    \underset{\nu>0, \eta \in \mathbb{R}}{\max}\nu \big[\eta-\big(\left(1-\theta_1\right)f^{*}\left(\eta+ \epsilon-\frac{a}{\nu}\right) \\
     +\sum^n_{i=1}(g_{i+1}-g_{i})\theta_if^{*}\left(\eta+ \epsilon-\frac{g_i}{\nu}\right)+\theta_nf^{*}\left(\eta+\nu \epsilon-\frac{g_n}{\nu}\right)\big)\big]
\end{array}
\end{equation*}
   \end{theorem}

Therefore, to compare our algorithm with the CDF-based algorithm in \cite{kumar2021policy}, let $H(g_i)=P(J(\tau)\geq g_i)$ represent the probability of the cumulative reward is above $s_i$, we have the following derivation:
\begin{equation*}
    \begin{array}{l}
    \eta-(\left(1-\theta_1\right)f^{*}\left(\eta+ \epsilon-\frac{a}{\nu}\right) +\\
    \quad \sum^n_{i=1}(g_{i+1}-g_{i})\theta_if^{*}\left(\eta+ \epsilon-\frac{g_i}{\nu}\right)+\theta_nf^{*}\left(\eta+ \epsilon-\frac{g_n}{\nu}\right))\\
    =\eta - \underset{\tau \sim p}E[f^{*}\left(\eta+\epsilon-\frac{a}{\nu}\right)H(a)+\sum^n_{i=1}f^{*}\left(\eta+\epsilon-\frac{g_i}{\nu}\right)H(g_i)\\
    \quad +f^{*}\left(\eta+ \epsilon-\frac{g_n}{\nu}\right)H(g_n)]\\ 
    \leq \eta - \underset{\tau \sim p}{\mathrm{E}}[f^{*}(\left(\eta+ \epsilon-\frac{a}{\nu}\right)H(a)+\sum^n_{i=1}\left(\eta+ \epsilon-\frac{g_i}{\nu}\right)H(g_i)+\\
    \quad \left(\eta+ \epsilon-\frac{g_i}{\nu}\right)H(g_n))]\\
    =\eta - \underset{\tau \sim p}{\mathrm{E}}[f^{*}(\left(\eta+ \epsilon\right)(H(a)+\\
    \quad \sum^n_iH(g_i))-\frac{a}{\nu}H(a)-\sum^n_{i=1} \frac{g_i}{\nu}H(g_i))]\\
    =\eta -\underset{\tau \sim p}{\mathrm{E}}[f^{*}(\eta+ \epsilon-\frac{a}{\nu}H(a)-\sum^n_{i=1} \frac{g_i}{\nu}H(g_i))]\\
    =\eta -\underset{\tau \sim p}{\mathrm{E}}[f^{*}(\eta+ \epsilon-\frac{J(\tau)}{\nu})]
\end{array}
\end{equation*}
The inequality in the derivation follows the Jensen's inequality. Therefore, by optimising the Eq. \ref{eq:opt}, it is possible to achieve either equivalent or more stringent bounds as those in \cite{kumar2021policy}. The experimental optimisation objective aimed at certifying the $l_2$-norm perturbation is detailed below:
\begin{theorem}
(Optimisation Objective to verify $l_2$-norm bound by Gaussian Noise ) Given the parameter $\lambda$ of the Hockey-Stick divergence, we can solve following convex optimisation problem to find the minimum bound of the mean cumulative reward:
    \begin{equation}\small
        \begin{array}{cc}
             \underset{\nu>0, \eta \in \mathbb{R}}{\max} \{\eta -\underset{\tau \sim p}{\mathrm{E}}[\max((\eta+\nu \epsilon-J(\tau))\lambda,0)+\nu \max(1-\lambda,0)]\}\\
             s.t \quad \eta \leq \nu(1-\epsilon)+J(\tau)
        \end{array}
    \end{equation}
\end{theorem}

\subsection{Certification on $l_1$-norm}
To certify the $l_1$-norm perturbation, the problem in the context of RL can be formulated as the accumulation of perturbation magnitudes applied to the agent's observations,  constrained by $\epsilon_D$. The objective is to find $\epsilon_D$ satisfies $\mathcal{D}_{\epsilon,s_t}: =\{\mu(s_t+\delta): ||\delta||_1\leq \epsilon\} \in \mathcal{D}_{\epsilon_D,s_t}:=\{q\in \mathcal{P}(s_t):D_f(q||p)\leq\epsilon_D\}$. Suppose we draw noise from the Gaussian distribution $\mathcal{N}(0,\sigma^2I_D)$, then we can adapt Theorem I in \cite{estgu}. This allows us to employ the total variation distance $TV(q||p) = D_f(q||p)$, where $f(x) = \frac{1}{2}|x-1|$, as a measurement based on the $l_1$-norm.
\begin{theorem}
    Given state $s_t \in \mathcal{S}$, we consider two Gaussian distributions: $p = \mathcal{N}(s_t, \sigma^2 I_D)$ and $q = \mathcal{N}(s'_t, \sigma^2 I_D)$, where $s'_t = s_t+\delta_t$. Under the constraint of $||(\delta_1,\delta_2,...)||_1 \leq \epsilon$, the total variation distance for any two distributions $(p,q)$ is defined as $\epsilon_{TV}=2\Phi(\frac{\epsilon}{2\sigma^2}-1)$, where $\Phi$ is the CDF of standard norm distribution. 

    The objective function to certify the lower bound of the expected utility under $l_1$-norm perturbationis:
\begin{equation}
        \begin{array}{cc}
         \underset{\nu>0, \eta \in \mathbb{R}}{\max} \left\{\eta-\underset{\tau \sim p}{\mathrm{E}}\left[\max(\eta+\nu\epsilon-J(\tau),-\frac{\nu}{2})\right]\right\}\\
             s.t  \quad \eta \leq \nu(\frac{1}{2}-\epsilon)+ \min(J(\tau))
        \end{array}
    \end{equation}    
\end{theorem}

\begin{figure*}
    \centering
    \includegraphics[width=0.4\linewidth]{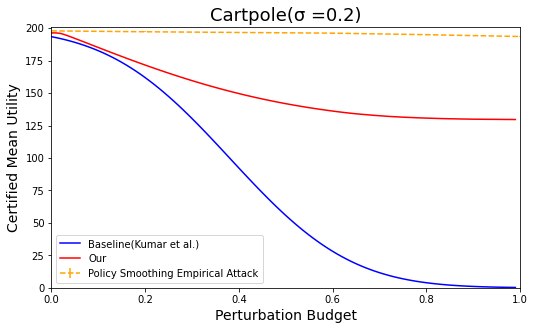}
    \includegraphics[width=0.4\linewidth]{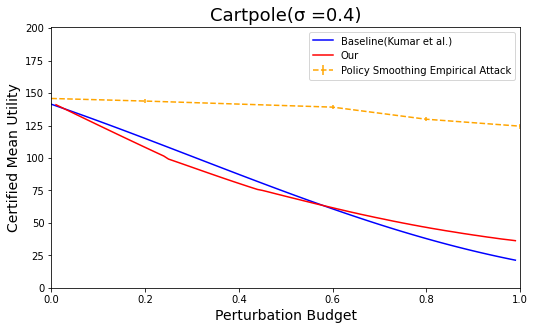}
    \\
    \includegraphics[width=0.4\linewidth]{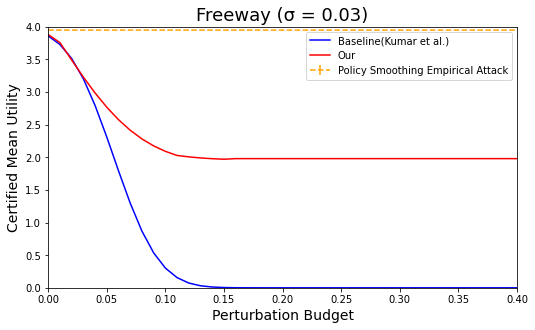}
    \includegraphics[width=0.4\linewidth]{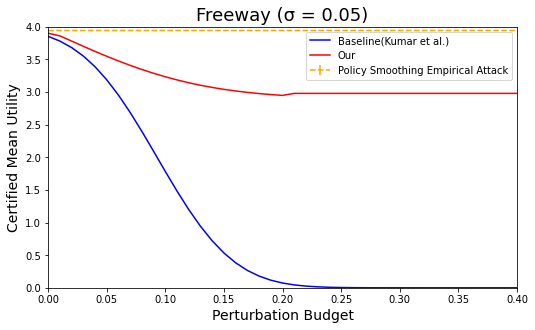}
    \caption{Comparison of the certified lower bound of mean utility under $l_2$-norm perturbation (where higher certified utility is better). For 'Cartpole', the baseline has a runtime of 0.039s per budget, while ours is 0.016s. For 'Freeway', both our method and the baseline exhibit similar running times, both at 0.03 s per budget. The dashed line represents the empirical attack result.}
    \label{fig:enter-label}
\end{figure*}
\subsection{Certification on $l_0$-norm} \label{sec:l0}
Our approach can also be extended to certify perturbations bounded by $l_0$-norm. Within the RL domain, the $l_0$ norm can be employed to quantify the number of actions manipulated when performing attacks in the discrete action space. 
We define a smoothed policy denoted as $\tilde \pi$, which is developed by introducing an element of randomness by not selecting the best action at each step with probability $k$. For the certification, we choose the worst action with a probability of $k$ at each step. The objective is to establish a lower bound for the mean cumulative reward under an $l_0$-norm bound, which specifies the maximum number of allowed actions that can be changed. In this paper, we use the Rényi divergence with parameter $\beta$ to measure the $l_0$ perturbation, and the optimisation objective can be formulated as follows:
\begin{theorem}\label{the:5}

Given a trained smoothed policy $\tilde \pi$ with action space $\mathcal{A}:\{a_1,...,a_N\}$. Suppose at each step the policy has a probability of $k$ to select the best action $a^*$, and a probability of $e$ to select an action different from $a^*$. The total number of time steps is denoted as $T$. The action sequences $\mathcal{A}$ that selected by the smoothed policy can be viewed as on the distribution of $\mu(\mathcal{A})=\prod_{i=1}^{T} (1-e)^{1\left[a_{i}=a^*_{i}\right]}\left(\frac{e}{N-1}\right)^{1\left[a_{i} \neq a^*_{i}\right]}$. The certification objective is:
$\underset{\nu>0, \eta \in \mathbb{R}}{\max} (\nu\eta-\underset{\tau \sim p}{\mathrm{E}}\left[\mathcal{Z}\right])$,
where 
\begin{equation*}
\mathcal{Z}=\left\{
\begin{array}{cc}
\nu+\nu^{\frac{-1}{\beta-1}}(\beta-1)\left(\frac{\max (x, 0)}{ \beta}\right)^{\frac{\beta}{\beta-1}} \quad if \quad \beta > 1\\
-\nu+\nu^{\frac{-1}{\beta-1}}(1-\beta)\left(\frac{x}{ -\beta}\right)^{\frac{\beta}{\beta-1}} \quad if \quad 0\leq\beta < 1, x\leq 0
\end{array}
\right.
\end{equation*}
where $x=\nu\eta+\nu\epsilon-J(\tau)$.
\end{theorem}

\section{Experiments}
We first demonstrate the experimental results of our method - {\bf ReCePS}
\footnote{Our code is available at \url{https://github.com/TrustAI/ReCePS}}, compared with the baseline to certify the cumulative reward under $l_2$-norm bounded perturbation. Then, we show the results of certifying the $l_1$-norm perturbation on observations and the $l_0$-norm perturbation on action space.

\textbf{Environments} We followed the baseline \cite{kumar2021policy} to apply our algorithm on three standard environments: two classic control environments (`Cartpole' and 'Mountain Car') and one high-dimensional Atari game (`Freeway'). Due to the page limit, we present the results for `Cartpole' and `Freeway' in the main paper, and extra experiments are provided in Appendix I. 

\textbf{RL algorithm} Among the environments, the `Cartpole' and two Atrari games adapt a discrete action, and `Mountain Car' utilises a continuous action space. We use the Deep-Q-Network (DQN) \cite{mnih2013playing} to train the policy for the discrete action space and we use deep Deterministic Gradient (DDPG) \cite{lillicrap2015continuous} to train the policy based on continuous action space (Details are in Appendix). 

\textbf{Environments setup} In all experiments, we initially select $M=1000$ to determine the optimal value in the optimisation function. Subsequently, we set $\tilde M=10000$ to derive the lower bound for the mean cumulative reward. The confidence level for all experimental outcomes, including baselines, is set at $1-\alpha =0.99$.

\subsection{Certified bound under $l_2$ and $l_1$ perturbations}
In scenarios where the adversary can influence the agent's observations, we adopt the approach from prior research, assuming that the adversary intervenes in each frame only once, right when it is initially observed. For both the training and testing phases, we introduce noise to create a defence policy based on smoothing techniques.

For the certification of $l_2$-norm perturbation, we compare our method with \cite{kumar2021policy}, which is also designed to certify the adaptive RL adversary and follow the same setting to make a fair comparison. \cite{wu2021crop} also proposed an algorithm to certify the cumulative reward, but they focused on certifying the policy at each step against the non-adaptive adversary. Their global reward certification method based on Lipchitz continuity is too weak compared with \cite{kumar2021policy} and our method (Results presented in Appendix I).
\begin{figure}
    \centering
    \includegraphics[width=1\linewidth]{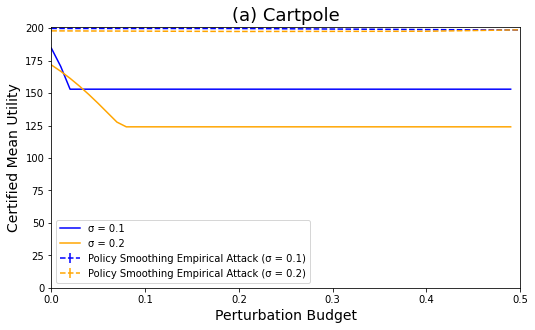}\\
    \includegraphics[width=1\linewidth]{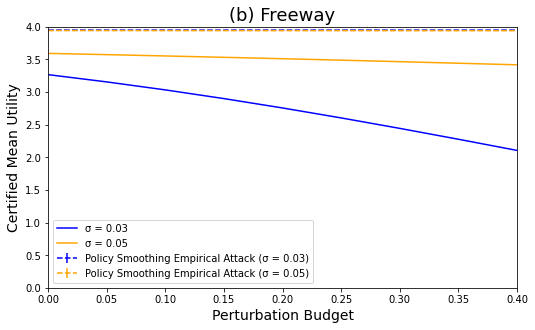}
   \caption{Certified lower bound of mean cumulative reward, under $l_1$-norm bounded perturbation.}
    \label{fig:l1}
    \vspace{-3mm}
\end{figure}
We use the Hocky-Stick divergence to measure the norm distance and tune the parameter $\lambda$ to achieve the best certificate. As shown in Figure \ref{fig:enter-label}, our method can obtain a tighter certified bound when we increase the perturbation budget, especially for the `Freeway' environment. However, we can also observe that when we increase the smoothing parameter $\sigma$, the bound of our method gets closer to the bound obtained by the baseline. This could occur because as the $\sigma$ grows, the variance of the cumulative output reward distribution also increases. Consequently, employing the mean across several simulations to estimate the expectation term might result in inaccuracies when solving the optimisation problem. 

Concerning the certification of the $l_1$ norm perturbation, we present the results in Figure \ref{fig:l1}. These perturbations are structured as the sum of perturbations applied to all states, assuming that the adversary introduces perturbations to state observations throughout the entire episode. We follow the same procedure of \cite{kumar2021policy} to perform the attack. The outcomes reveal a clear trend: elevating the $\sigma$ value enhances policy robustness, but excessive $\sigma$ values can lead to reduced performance (lower reward).
\begin{figure}
    \centering
    \includegraphics[width=1\linewidth]{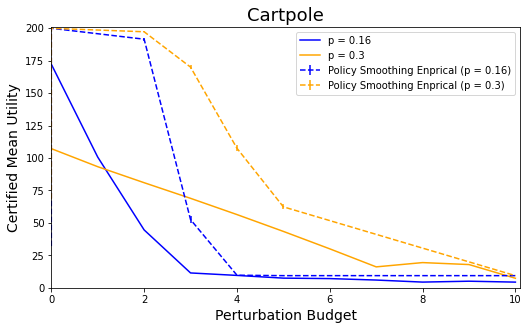}
   \caption{Certified lower bound and empirical attack result of mean cumulative reward, under $l_0$-norm bounded perturbation on action space in `Cartpole' environment with two discrete actions. $p$ is the probability of changing the action.}
    \label{fig:l0}
  \vspace{-3mm}
\end{figure}
\subsection{Certified bound under $l_0$ perturbations}
 For the perturbation on action space, we adapted the attack method in \cite{LinHLSLS17}, which performs an attack when the gap between the value of the best and worst actions is above a threshold. To defend against such attacks, during training and testing time we will enforce the agent to select the worst action with probability $p$. By adopting the Rényi divergence, we can compute the lower bound of cumulative reward by finding the optimal value of the convex optimisation problem in Theorem \ref{the:5}. However, as the problem is non-linear, we use the `SciPy' optimise Python package to solve it, and we tune the parameter of Rényi divergence, $\beta$, to obtain the optimal certificate. We present the results of the experiment in `Cartpole' environment in Figure \ref{fig:l0} and experiments on other environments can be found in Appendix I. It's evident that increasing the smoothing parameter $p$ enhances the policy's robustness against attacks. 

\section{Conclusion}

We introduce a general framework to provide a lower bound of mean cumulative reward certification that can handle perturbations bounded by various norms. Our approach employs the $f$-divergence to quantify the magnitude of adversarial perturbation. By solving a convex optimisation problem, we compute the optimal minimum utility. Our assessments demonstrate that our approach yields tighter bounds compared to existing state-of-the-art techniques, extending its applicability to perturbations within the action space.
\section{Acknowledgement}
The research is supported by the UK EPSRC 
under project EnnCORE [EP/T026995/1].

\bibliography{aaai24}

\begin{thebibliography}{37}
\providecommand{\natexlab}[1]{#1}

\bibitem[{Barsov and Ulyanov(1987)}]{estgu}
Barsov, S.; and Ulyanov, V. 1987.
\newblock Estimates of the proximity of Gaussian measures.
\newblock \emph{Doklady Mathematics}, 34: 462--.

\bibitem[{Ben-Tal and Teboulle(2007)}]{ben2007old}
Ben-Tal, A.; and Teboulle, M. 2007.
\newblock An old-new concept of convex risk measures: The optimized certainty
  equivalent.
\newblock \emph{Mathematical Finance}, 17(3): 449--476.

\bibitem[{Cheng et~al.(2019)Cheng, Orosz, Murray, and Burdick}]{cheng2019end}
Cheng, R.; Orosz, G.; Murray, R.~M.; and Burdick, J.~W. 2019.
\newblock End-to-end safe reinforcement learning through barrier functions for
  safety-critical continuous control tasks.
\newblock In \emph{Proceedings of the AAAI conference on artificial
  intelligence}, volume~33, 3387--3395.

\bibitem[{Christiano et~al.(2016)Christiano, Shah, Mordatch, Schneider,
  Blackwell, Tobin, Abbeel, and Zaremba}]{christiano2016transfer}
Christiano, P.; Shah, Z.; Mordatch, I.; Schneider, J.; Blackwell, T.; Tobin,
  J.; Abbeel, P.; and Zaremba, W. 2016.
\newblock Transfer from simulation to real world through learning deep inverse
  dynamics model.
\newblock \emph{arXiv preprint arXiv:1610.03518}.

\bibitem[{Cohen, Rosenfeld, and Kolter(2019)}]{cohen2019certified}
Cohen, J.; Rosenfeld, E.; and Kolter, Z. 2019.
\newblock Certified adversarial robustness via randomized smoothing.
\newblock In \emph{international conference on machine learning}, 1310--1320.
  PMLR.

\bibitem[{Csisz{\'a}r(1967)}]{csiszar1967information}
Csisz{\'a}r, I. 1967.
\newblock Information-type measures of difference of probability distributions
  and indirect observation.
\newblock \emph{Studia Scientiarum Mathematicarum Hungarica}, 2: 299--318.

\bibitem[{Diamond and Boyd(2016)}]{diamond2016cvxpy}
Diamond, S.; and Boyd, S. 2016.
\newblock {CVXPY}: {A} {P}ython-embedded modeling language for convex
  optimization.
\newblock \emph{Journal of Machine Learning Research}, 17(83): 1--5.

\bibitem[{Dong et~al.(2021)Dong, Yang, Deng, Pang, Xiao, Su, and
  Zhu}]{dong2021black}
Dong, Y.; Yang, X.; Deng, Z.; Pang, T.; Xiao, Z.; Su, H.; and Zhu, J. 2021.
\newblock Black-box detection of backdoor attacks with limited information and
  data.
\newblock In \emph{Proceedings of the IEEE/CVF International Conference on
  Computer Vision}, 16482--16491.

\bibitem[{Ehlers(2017)}]{ehlers2017formal}
Ehlers, R. 2017.
\newblock Formal verification of piece-wise linear feed-forward neural
  networks.
\newblock In \emph{Automated Technology for Verification and Analysis: 15th
  International Symposium, ATVA 2017, Pune, India, October 3--6, 2017,
  Proceedings 15}, 269--286. Springer.

\bibitem[{Eysenbach and Levine(2022)}]{DBLP:conf/iclr/EysenbachL22}
Eysenbach, B.; and Levine, S. 2022.
\newblock Maximum Entropy {RL} (Provably) Solves Some Robust {RL} Problems.
\newblock In \emph{The Tenth International Conference on Learning
  Representations, {ICLR} 2022, Virtual Event, April 25-29, 2022}.
  OpenReview.net.

\bibitem[{Goodfellow, Shlens, and Szegedy(2015)}]{GoodfellowSS14}
Goodfellow, I.~J.; Shlens, J.; and Szegedy, C. 2015.
\newblock Explaining and Harnessing Adversarial Examples.
\newblock In Bengio, Y.; and LeCun, Y., eds., \emph{3rd International
  Conference on Learning Representations, {ICLR} 2015, San Diego, CA, USA, May
  7-9, 2015, Conference Track Proceedings}.

\bibitem[{Jin et~al.(2022)Jin, Yi, Huang, Schewe, and Huang}]{jin2022enhancing}
Jin, G.; Yi, X.; Huang, W.; Schewe, S.; and Huang, X. 2022.
\newblock Enhancing adversarial training with second-order statistics of
  weights.
\newblock In \emph{Proceedings of the IEEE/CVF Conference on Computer Vision
  and Pattern Recognition}, 15273--15283.

\bibitem[{Johannink et~al.(2019)Johannink, Bahl, Nair, Luo, Kumar, Loskyll,
  Ojea, Solowjow, and Levine}]{johannink2019residual}
Johannink, T.; Bahl, S.; Nair, A.; Luo, J.; Kumar, A.; Loskyll, M.; Ojea,
  J.~A.; Solowjow, E.; and Levine, S. 2019.
\newblock Residual reinforcement learning for robot control.
\newblock In \emph{2019 International Conference on Robotics and Automation
  (ICRA)}, 6023--6029. IEEE.

\bibitem[{Kumar, Levine, and Feizi(2021)}]{kumar2021policy}
Kumar, A.; Levine, A.; and Feizi, S. 2021.
\newblock Policy smoothing for provably robust reinforcement learning.
\newblock \emph{arXiv preprint arXiv:2106.11420}.

\bibitem[{Kumar et~al.(2020)Kumar, Levine, Feizi, and
  Goldstein}]{kumar2020certifying}
Kumar, A.; Levine, A.; Feizi, S.; and Goldstein, T. 2020.
\newblock Certifying confidence via randomized smoothing.
\newblock \emph{Advances in Neural Information Processing Systems}, 33:
  5165--5177.

\bibitem[{Lecuyer et~al.(2019)Lecuyer, Atlidakis, Geambasu, Hsu, and
  Jana}]{lecuyer2019certified}
Lecuyer, M.; Atlidakis, V.; Geambasu, R.; Hsu, D.; and Jana, S. 2019.
\newblock Certified robustness to adversarial examples with differential
  privacy.
\newblock In \emph{2019 IEEE symposium on security and privacy (SP)}, 656--672.
  IEEE.

\bibitem[{Lillicrap et~al.(2015)Lillicrap, Hunt, Pritzel, Heess, Erez, Tassa,
  Silver, and Wierstra}]{lillicrap2015continuous}
Lillicrap, T.~P.; Hunt, J.~J.; Pritzel, A.; Heess, N.; Erez, T.; Tassa, Y.;
  Silver, D.; and Wierstra, D. 2015.
\newblock Continuous control with deep reinforcement learning.
\newblock \emph{arXiv preprint arXiv:1509.02971}.

\bibitem[{Lin et~al.(2017)Lin, Hong, Liao, Shih, Liu, and Sun}]{LinHLSLS17}
Lin, Y.; Hong, Z.; Liao, Y.; Shih, M.; Liu, M.; and Sun, M. 2017.
\newblock Tactics of Adversarial Attack on Deep Reinforcement Learning Agents.
\newblock In Sierra, C., ed., \emph{Proceedings of the Twenty-Sixth
  International Joint Conference on Artificial Intelligence, {IJCAI} 2017,
  Melbourne, Australia, August 19-25, 2017}, 3756--3762. ijcai.org.

\bibitem[{Madry et~al.(2017)Madry, Makelov, Schmidt, Tsipras, and
  Vladu}]{MadryMSTV17}
Madry, A.; Makelov, A.; Schmidt, L.; Tsipras, D.; and Vladu, A. 2017.
\newblock Towards Deep Learning Models Resistant to Adversarial Attacks.
\newblock \emph{CoRR}, abs/1706.06083.

\bibitem[{Manikandan et~al.(2011)}]{manikandan2011measures}
Manikandan, S.; et~al. 2011.
\newblock Measures of central tendency: Median and mode.
\newblock \emph{J Pharmacol Pharmacother}, 2(3): 214--215.

\bibitem[{Mnih et~al.(2013)Mnih, Kavukcuoglu, Silver, Graves, Antonoglou,
  Wierstra, and Riedmiller}]{mnih2013playing}
Mnih, V.; Kavukcuoglu, K.; Silver, D.; Graves, A.; Antonoglou, I.; Wierstra,
  D.; and Riedmiller, M. 2013.
\newblock Playing atari with deep reinforcement learning.
\newblock \emph{arXiv preprint arXiv:1312.5602}.

\bibitem[{Mu et~al.(2023)Mu, Ruan, Marcolino, Jin, and Ni}]{mu2023certified}
Mu, R.; Ruan, W.; Marcolino, L.~S.; Jin, G.; and Ni, Q. 2023.
\newblock Certified policy smoothing for cooperative multi-agent reinforcement
  learning.
\newblock In \emph{Proceedings of the AAAI Conference on Artificial
  Intelligence}, volume~37, 15046--15054.

\bibitem[{Mu et~al.(2021)Mu, Ruan, Marcolino, and Ni}]{mu2021sparse}
Mu, R.; Ruan, W.; Marcolino, L.~S.; and Ni, Q. 2021.
\newblock Sparse adversarial video attacks with spatial transformations.
\newblock \emph{arXiv preprint arXiv:2111.05468}.

\bibitem[{Mu et~al.(2022)Mu, Ruan, Marcolino, and Ni}]{mu20223dverifier}
Mu, R.; Ruan, W.; Marcolino, L.~S.; and Ni, Q. 2022.
\newblock 3DVerifier: efficient robustness verification for 3D point cloud
  models.
\newblock \emph{Machine Learning}, 1--28.

\bibitem[{Pan et~al.(2017)Pan, You, Wang, and Lu}]{pan2017virtual}
Pan, X.; You, Y.; Wang, Z.; and Lu, C. 2017.
\newblock Virtual to real reinforcement learning for autonomous driving.
\newblock \emph{arXiv preprint arXiv:1704.03952}.

\bibitem[{Pattanaik et~al.(2017)Pattanaik, Tang, Liu, Bommannan, and
  Chowdhary}]{pattanaik2017robust}
Pattanaik, A.; Tang, Z.; Liu, S.; Bommannan, G.; and Chowdhary, G. 2017.
\newblock Robust deep reinforcement learning with adversarial attacks.
\newblock \emph{arXiv preprint arXiv:1712.03632}.

\bibitem[{Pattanaik et~al.(2018)Pattanaik, Tang, Liu, Bommannan, and
  Chowdhary}]{PattanaikTLBC18}
Pattanaik, A.; Tang, Z.; Liu, S.; Bommannan, G.; and Chowdhary, G. 2018.
\newblock Robust Deep Reinforcement Learning with Adversarial Attacks.
\newblock In Andr{\'{e}}, E.; Koenig, S.; Dastani, M.; and Sukthankar, G.,
  eds., \emph{Proceedings of the 17th International Conference on Autonomous
  Agents and MultiAgent Systems, {AAMAS} 2018, Stockholm, Sweden, July 10-15,
  2018}, 2040--2042. International Foundation for Autonomous Agents and
  Multiagent Systems Richland, SC, {USA} / {ACM}.

\bibitem[{Russo and Proutiere(2019)}]{russo2019optimal}
Russo, A.; and Proutiere, A. 2019.
\newblock Optimal attacks on reinforcement learning policies.
\newblock \emph{arXiv preprint arXiv:1907.13548}.

\bibitem[{Sallab et~al.(2017)Sallab, Abdou, Perot, and
  Yogamani}]{sallab2017deep}
Sallab, A.~E.; Abdou, M.; Perot, E.; and Yogamani, S. 2017.
\newblock Deep reinforcement learning framework for autonomous driving.
\newblock \emph{arXiv preprint arXiv:1704.02532}.

\bibitem[{Shen et~al.(2020)Shen, Li, Jiang, Wang, and Zhao}]{shen2020deep}
Shen, Q.; Li, Y.; Jiang, H.; Wang, Z.; and Zhao, T. 2020.
\newblock Deep reinforcement learning with robust and smooth policy.
\newblock In \emph{International Conference on Machine Learning}, 8707--8718.
  PMLR.

\bibitem[{Shivaswamy and Jebara(2010)}]{shivaswamy2010empirical}
Shivaswamy, P.; and Jebara, T. 2010.
\newblock Empirical bernstein boosting.
\newblock In \emph{Proceedings of the Thirteenth International Conference on
  Artificial Intelligence and Statistics}, 733--740. JMLR Workshop and
  Conference Proceedings.

\bibitem[{Szegedy et~al.(2014)Szegedy, Zaremba, Sutskever, Bruna, Erhan,
  Goodfellow, and Fergus}]{SzegedyZSBEGF13}
Szegedy, C.; Zaremba, W.; Sutskever, I.; Bruna, J.; Erhan, D.; Goodfellow,
  I.~J.; and Fergus, R. 2014.
\newblock Intriguing properties of neural networks.
\newblock In Bengio, Y.; and LeCun, Y., eds., \emph{2nd International
  Conference on Learning Representations, {ICLR} 2014, Banff, AB, Canada, April
  14-16, 2014, Conference Track Proceedings}.

\bibitem[{Wong and Kolter(2018)}]{wong2018provable}
Wong, E.; and Kolter, Z. 2018.
\newblock Provable defenses against adversarial examples via the convex outer
  adversarial polytope.
\newblock In \emph{International conference on machine learning}, 5286--5295.
  PMLR.

\bibitem[{Wu et~al.(2021)Wu, Li, Huang, Vorobeychik, Zhao, and Li}]{wu2021crop}
Wu, F.; Li, L.; Huang, Z.; Vorobeychik, Y.; Zhao, D.; and Li, B. 2021.
\newblock Crop: Certifying robust policies for reinforcement learning through
  functional smoothing.
\newblock \emph{arXiv preprint arXiv:2106.09292}.

\bibitem[{Zhang, Ruan, and Xu(2023)}]{zhang2023reachability}
Zhang, C.; Ruan, W.; and Xu, P. 2023.
\newblock Reachability analysis of neural network control systems.
\newblock \emph{arXiv preprint arXiv:2301.12100}.

\bibitem[{Zhang et~al.(2021)Zhang, Avrithis, Furon, and Amsaleg}]{9186644}
Zhang, H.; Avrithis, Y.; Furon, T.; and Amsaleg, L. 2021.
\newblock Walking on the Edge: Fast, Low-Distortion Adversarial Examples.
\newblock \emph{IEEE Transactions on Information Forensics and Security}, 16:
  701--713.

\bibitem[{Zhang et~al.(2020)Zhang, Chen, Xiao, Li, Liu, Boning, and
  Hsieh}]{zhang2020robust}
Zhang, H.; Chen, H.; Xiao, C.; Li, B.; Liu, M.; Boning, D.; and Hsieh, C.-J.
  2020.
\newblock Robust deep reinforcement learning against adversarial perturbations
  on state observations.
\newblock \emph{Advances in Neural Information Processing Systems}, 33:
  21024--21037.

\end{thebibliography}

\end{document}